\def\year{2020}\relax
\documentclass[letterpaper]{article} 
\usepackage{aaai20}  
\usepackage{times}  
\usepackage{helvet} 
\usepackage{courier}  
\usepackage[hyphens]{url}  
\usepackage{graphicx} 
\usepackage{graphicx}  
\usepackage{amsthm,amsmath,amssymb} 
\usepackage{todonotes}

\newtheorem{theorem}{Theorem}
\newtheorem{lemma}{Lemma}
\newtheorem{proposition}{Proposition}

\theoremstyle{definition}
\newtheorem{definition}{Definition}

\newcommand{\domain}{\ensuremath{\mathcal{D}}}
\newcommand{\arguments}{\ensuremath{\mathcal{A}}}
\newcommand{\attacker}{\ensuremath{\mathrm{Att}}}
\newcommand{\supporter}{\ensuremath{\mathrm{Sup}}}
\newcommand{\baseScore}{\ensuremath{\beta}}
\newcommand{\strength}{\ensuremath{\sigma}}
\newcommand{\strengthMLP}{\ensuremath{\sigma_{\textrm{MLP}}}}
\newcommand{\strengthcMLP}{\ensuremath{\sigma_{\textrm{cMLP}}}}
\newcommand{\discMLP}{\ensuremath{u_{\textrm{MLP}}}}
\newcommand{\contMLP}{\ensuremath{f^{\textrm{MLP}}}}

\newcommand{\diff}[1]{\ensuremath{\frac{\mathrm{d}#1}{\mathrm{d}t}}}

\title{Interpreting Neural Networks as Quantitative Argumentation Frameworks }
\author{
Nico Potyka \\
University of Stuttgart\\ 
Universitätsstraße 32\\
70569 Stuttgart, Germany\\
nico.potyka@ipvs.uni-stuttgart.de
}
\begin{document}

\maketitle

\begin{abstract}
We show that an interesting class of feed-forward neural networks can be understood as quantitative argumentation frameworks. 
This connection creates a bridge between research in Formal Argumentation and Machine Learning.
We generalize the semantics of feed-forward neural networks to acyclic graphs and study the resulting
computational and semantical properties in argumentation
graphs. As it turns out, the semantics gives stronger guarantees than existing semantics 
that have been tailor-made for the argumentation setting. 
From a machine-learning perspective, the connection does not seem immediately helpful. While it gives intuitive meaning
to some feed-forward-neural networks, they remain difficult to understand due to their size and density.
However, the connection seems helpful for combining background knowledge in form of
sparse argumentation networks with dense neural networks that have been trained for complementary purposes
and for learning the parameters of quantitative argumentation frameworks in an end-to-end fashion from data.
\end{abstract}

\section{Introduction}

In this paper, we establish a relationship between neural networks and abstract argumentation frameworks. 
More precisely, we study relationships between quantitative bipolar argumentation frameworks (QBAFs) and multilayer perceptrons (MLPs).
QBAFs are a knowledge representation formalism that can be used to solve 
decision problems in a very intuitive way by weighing up pro and contra arguments \cite{baroni2015automatic,rago2016discontinuity,amgoud2017evaluation}.
QBAFs and their variants have been combined with machine learning methods in order to add explainability to problems like
product recommendation \cite{rago2018argumentation}, review aggregation \cite{cocarascu2019extracting} and stance aggregation in fake news detection \cite{kotonya2019gradual}.
Multilayer perceptrons (MLPs) \cite{goodfellow2016deep} are a very flexible class of feed-forward neural networks that can be applied
in basically all machine learning tasks. This includes applications like
classification \cite{heidari2019efficient}, regression \cite{hiransha2018nse} and function approximation in reinforcement learning \cite{tesauro1995temporal}. 

We explain the basics of QBAFs and MLPs in Sections \ref{sec_QA_basics} and \ref{sec_MLP_basics}, respectively.
In Section \ref{sec_MLP_semantics}, we introduce an MLP-based semantics for QBAFs that is
based on computing the strength of arguments in an iterative way. In acyclic graphs, the result is 
equal to the result of the usual evaluation procedure (forward propagation) for MLPs. We give sufficient
conditions for convergence of this procedure in cyclic graphs and analyze the convergence rate.
Simply put, convergence is guaranteed when the edge weights and the indegree of arguments is not too large.
We give an example that demonstrates that our convergence conditions cannot be improved without adding additional
assumptions about the structure of the graph. In order to improve the guarantees, we introduce
a continuous variant that agrees with its discrete counterpart in the known convergence cases, but still
converges in more general cases. Finally, we show that the MLP-based semantics satisfies all properties
for QBAF semantics proposed in \cite{amgoud2017evaluation,potyka2018Kr,Potyka19om} almost perfectly. 
This is surprising because it actually gives
stronger semantical guarantees than some semantics that have been designed specifically for QBAFs.
We close the paper with some ideas about how this relationship can be exploited to
combine ideas for  QBAFs and neural networks fruitfully for both fields. 

\section{QBAF Basics}
\label{sec_QA_basics}

In this work, our conceptual understanding of an argument follows Dung's notion of \emph{abstract argumentation}:
"an argument is an abstract entity whose role is solely determined by its relations to other arguments" \cite{dung1995acceptability}. 
That is, we abstract from the content of arguments and focus on their acceptability dependent on
the acceptability of their attackers and supporters.
This idea can be formalized in different ways, we refer to \cite{baroni2018abstract} for an overview of some classical approaches. 
Here, we consider quantitative bipolar argumentation frameworks (QBAFs) similar to \cite{baroni2018many}.
In general, these frameworks interpret arguments by values from an arbitrary domain $\domain$. 
For simplicity, we assume that $\domain = [0,1]$. Intuitively, the value $0$ means that an argument
is fully rejected, $1$ means that it is fully accepted and values in between balance between these extremes.
\begin{definition}[QBAF]
A QBAF (over $\domain = [0,1]$) is a quadruple
$(\arguments, \attacker, \supporter, \baseScore)$ consisting of 
a set of arguments $\arguments$, two binary relations
$\attacker$ and $\supporter$ called attack and support 
and a function $\baseScore: \arguments \rightarrow [0,1]$
that assigns a \emph{base score} $\baseScore(a)$
to every argument $a \in \arguments$.
\end{definition}
The base score can be seen as an apriori strength of an argument when it is evaluated independent of its relationships to other arguments.
This apriori strength will be adapted dynamically based on the strength of its attackers and supporters.
Graphically, we denote attack relations by solid and support relations by dashed edges
as illustrated in Figure \ref{fig:fig_QBAF} on the left.
\begin{figure}[tb]
	\centering
		\includegraphics[width=0.43\textwidth]{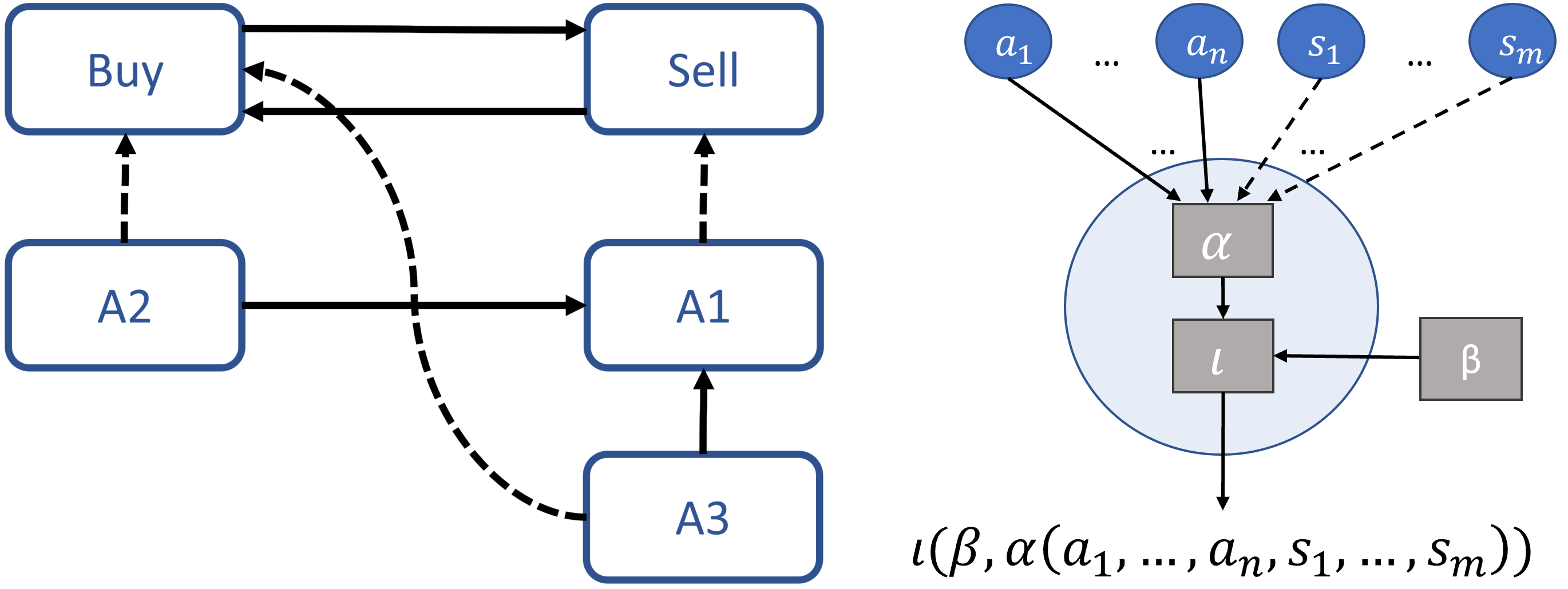}
	\caption{Example of a QBAF (left) and illustration of local update mechanics under modular semantics (right).}
	\label{fig:fig_QBAF}
\end{figure}
The QBAF models part of a decision problem from \cite{Potyka2018_tut}, where we want to decide whether to buy new or sell existing stocks of a company. 
A1 corresponds to the statement of an expert that recommends selling. A2 and A3 correspond to statements by experts who contradict the premises of A1 
and recommend buying. The selling and the buying decision are simply modeled as arguments that attack each other, 
so that the confidence in one decision will decrease the confidence in the other.

The main computational problem in QBAFs is to assign a strength value to arguments. 
We describe this process by interpretations. 
\begin{definition}[QBAF interpretation]
\label{def_QBAF_interpretation}
Let $Q $ be a QBAF over $[0,1]$.
An interpretation of $Q$ is a function $\strength: \arguments \rightarrow [0,1] \cup \{\bot\}$
and $\strength(a)$ is called the strength of $a$ for all $a \in \arguments$.
If $\strength(a) = \bot$ for some $a \in \arguments$, $\strength$ is called \emph{partial}. Otherwise,
it is called \emph{fully defined}.
\end{definition}
\emph{Modular semantics} define interpretations based on an iterative procedure \cite{mossakowski2018modular}.
For every argument, its strength is initialized with its base score. The strength values are then adapted
iteratively by applying an \emph{aggregation function} $\alpha$ and an \emph{influence function} $\iota$
as illustrated in Figure \ref{fig:fig_QBAF} on the right. The aggregation function $\alpha$ aggregates the strength
values of attackers and supporters. Aggregation functions have been based on product \cite{baroni2015automatic,rago2016discontinuity},
addition \cite{amgoud2017evaluation,potyka2018Kr} and maximum \cite{mossakowski2018modular}.
The influence function then takes the aggregate and the base score in order
to determine a new strength from the desired domain. Intuitively, supporters increase the strength, while attackers decrease it.
If the strength values converge, the limit defines the final strength value. Otherwise, strength values remain undefined
and the interpretation is partial. Of course, it would be desirable to always have fully defined interpretations. However, as
shown in \cite{mossakowski2018modular}, many update procedures can fail to converge in cyclic QBAFs. 
Properties for evaluating and comparing different semantics have been discussed in \cite{amgoud2017evaluation,baroni2018many,potyka2018Kr}.
We will explain these properties in detail later when we analyze neural networks as QBAFs.

\section{MLP Basics}
\label{sec_MLP_basics}

Intuitively, a multilayer perceptron (MLP) is a layered acyclic graph as sketched in Figure \ref{fig:fig_nns} on the left.
Formally, we describe MLPs as follows.
\begin{figure}
	\centering
		\includegraphics[width=0.45\textwidth]{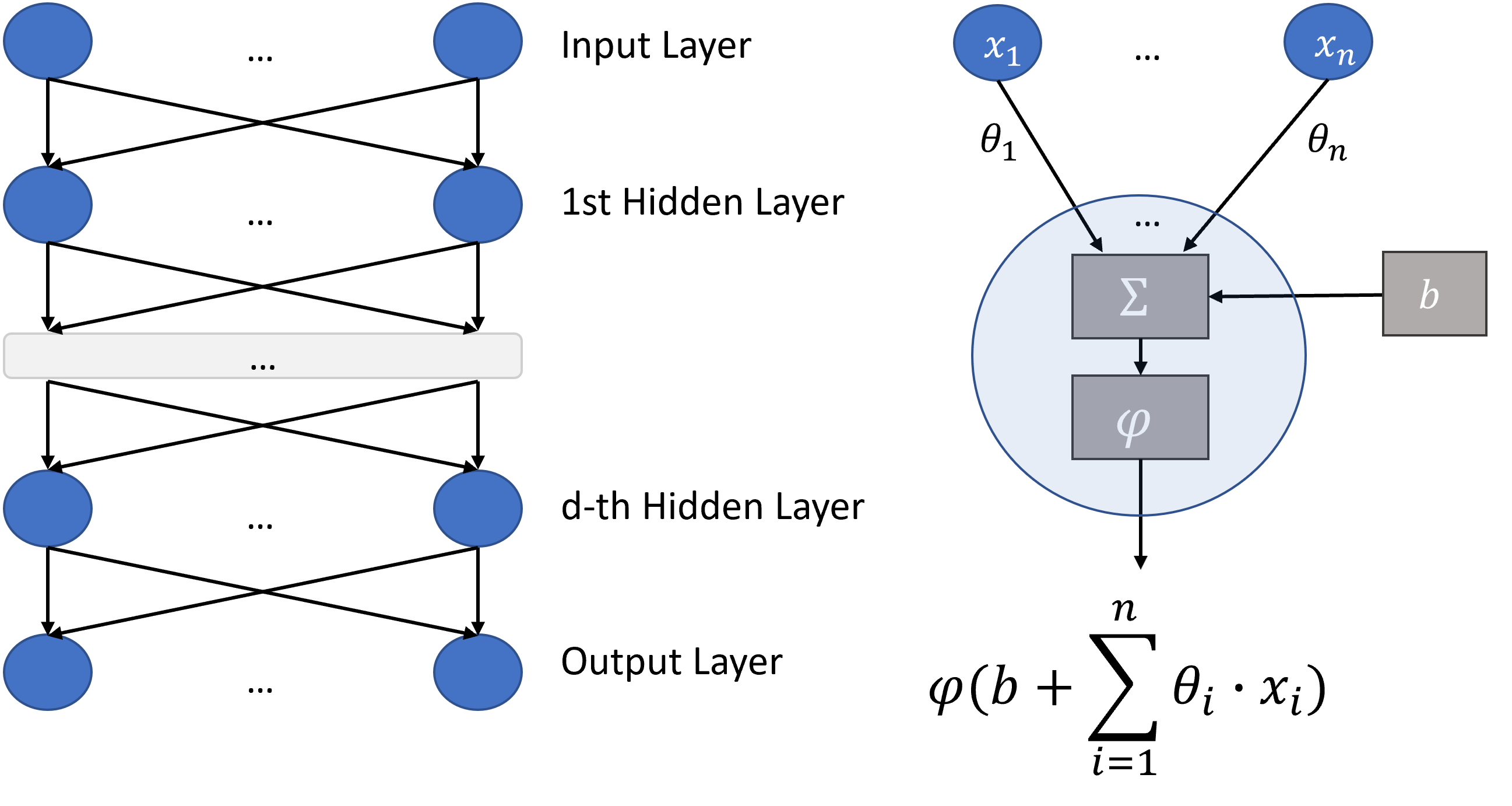}
	\caption{Graphical structure of an MLP (left) and illustration of local forward propagation (right). }
	\label{fig:fig_nns}
\end{figure}
\begin{definition}[MLP]
An MLP is a tuple $(V, E, B, \Theta)$, where 
\begin{itemize}
  \item $(V,E)$ is a directed graph. 
	\item $V = \uplus_{i=0}^{d+1} V_i$ is the disjoint union of sets of nodes $V_i$.
	\item We call $V_0$ the \emph{input layer}, $V_{d+1}$ the \emph{output layer} and $V_i$ the $i$-th \emph{hidden layer} for $1\leq i \leq d$.
	\item We call $d$ the depth of the network.
	\item $E \subseteq \bigcup_{i=0}^d \big(V_i \times V_{i+1} \big)$ is a set of edges between subsequent layers.
	 If $E = \bigcup_{i=0}^d \big(V_i \times V_{i+1} \big)$, the network is called \emph{fully connected}.
	\item $B: (V \setminus V_0) \rightarrow \mathbb{R}$ assigns a \emph{bias} to every non-input node.
	\item $\Theta: V \rightarrow \mathbb{R}$ assigns a \emph{weight} to every edge.
\end{itemize}
\end{definition}
In order to process an example, the input layer of an MLP is initialized with feature values of the example. 
These inputs are then propagated forward through the network
to generate an output in the output layer. For example, in a binary classification task, the output layer could consist 
of a single node whose value corresponds to the model's confidence that the example belongs to the class. 
The values at nodes in hidden layers and the output layer are computed by propagating the values from the input layer forward
through the network as sketched in Figure \ref{fig:fig_nns} on the right. 
Every edge is associated with a weight. For every ingoing edge $e_i$,
the corresponding weight $\theta_i=\Theta(e_i)$ is multiplied by the value $x_i$ of its source and the resulting values are summed up. 
The bias $b = B(v_j)$ of the edge's target $v_j$ is added and the result is fed into an activaction function $\varphi$. 
A popular choice to obtain values between $0$ and $1$ is the logistic activation function that is
defined by $\varphi_l(z) = \frac{1}{1 + \exp(-z)}$. The logistic function lost popularity since it can slow down gradient-based
training due to vanishing derivatives close to $0$ and $1$. However, recent ideas like batch normalization \cite{IoffeS15}
can mitigate the problem. In principle, the following ideas can be applied to other activation functions
like rectified linear units as well. However, values between $0$ and $1$ yield a particularly nice and simple interpretation.
We will therefore focus on logistic activation functions in the following.

\section{MLP-based Semantics for QBAFs}
\label{sec_MLP_semantics}

When comparing the update mechanics of QBAFs as sketched in Figure \ref{fig:fig_QBAF} on the right
with the forward propagation mechanics of MLPs as sketched in Figure \ref{fig:fig_nns} on the right,
we see that they are very similar. Roughly speaking, we can view an MLP as a QBAF where the aggregation
function $\alpha$ is based on addition and the influence function $\iota$ is based on a neural network activation function.
It is then natural to ask, does this QBAF give meaningful guarantees from an argumentation perspective?
In order to answer this question, we consider edge-weighted QBAFs as already considered in \cite{mossakowski2018modular}.
We consider only one set of edges and regard edges with negative weights as attacks and edges with positive weights as supports.
This simplifies making the connection between MLPs and QBAFs, but may not be appropriate in more general settings where
the aggregation function is not based on addition.
\begin{definition}[Edge-weighted QBAF]
An edge-weighted QBAF (over $\domain = [0,1]$) is a quadruple
$(\arguments, E, \baseScore, w)$ consisting of 
a set of arguments $\arguments$, edges $E \subseteq \arguments \times \arguments$ between these
arguments, a function $\baseScore: \arguments \rightarrow [0,1]$
that assigns a \emph{base score} $\baseScore(a)$
to every argument $a \in \arguments$ and a function $w: E \rightarrow \mathbb{R}$ that assigns a weight
to every edge.
\end{definition}
To simplify the presentation, we assume that $\arguments = \{1, 2, \dots, n\}$ in the following.
That is, the names of arguments correspond to numbers.
Furthermore, for every argument $a \in \arguments$, we let $\attacker(a) = \{(b,a) \in E \mid w(b,a) < 0\}$ and
$\supporter(a) = \{(b,a) \in E \mid w(b,a) > 0\}$.

In order to interpret the arguments in an edge-weighted QBAF, we consider a modular semantics based on the relationship
between QBAFs and MLPs noted earlier. 
The strength values are computed iteratively.
In every iteration, we have a strength vector $s^{(i)} \in [0,1]^n$. Its $a$-th element $s_a^{(i)}$ is the strength value
of argument $a$ in the $i$-th iteration. 
For every argument $a \in \arguments$, we let $s_a^{(0)} := \baseScore(a)$
be the initial strength value. The strength values are then updated by doing the following two steps repeatedly for all
$a \in \arguments$:
\begin{description}
\item[Aggregation:] We let $\alpha_a^{(i+1)} := \sum_{(b,a) \in E} w(b,a) \cdot s_b^{(i)}$.
\item[Influence:] We let $s_a^{(i+1)} := \varphi_l\big(\ln(\frac{\baseScore(a)}{1- \baseScore(a)}) + \alpha_a^{(i+1)} \big)$,
where $\varphi_l(z) = \frac{1}{1 + \exp(-z)}$ is the logistic function.
\end{description} 
Strictly speaking, the influence function is undefined for $\baseScore(a) \in \{0,1\}$. However, we can complete the definition by
using the infinite limits at these points. That is, we let $\ln(0) := - \infty$, $\ln(\frac{1}{0}) := \infty$,
$\varphi_l(-\infty)=0$, $\varphi_l(\infty)=1$ and for all $x \in \mathbb{R}$, $x - \infty = -\infty$ and $x + \infty = \infty$.
In this way, the composition of the aggregation and influence function is continuous and always returns values from the 
closed interval $[0,1]$.
By putting the definition of the aggregation function into the influence function, we obtain the explicit form of the update function $\discMLP: [0,1]^n \rightarrow [0,1]^n$ whose $i$-h component is defined by
\begin{equation}
\label{eq_discrete_update_formula}
\frac{1}{1 + \frac{1-\baseScore(i)}{\baseScore(i)}\exp(-\sum_{(b,i) \in E} w(b,i) \cdot s_b)}.
\end{equation}
Note that $s^{(k)} = \discMLP^k(s^{(0)})$, that is, $s^{(k)}$ is obtained from $s^{(0)}$ by applying $\discMLP$ $k$ times.
The MLP-based semantics is defined based on the result of applying the aggregation and influence function repeatedly.
\begin{definition}[MLP-based Semantics]
Let $Q $ be an edge-weighted QBAF over $[0,1]$. The interpretation of $Q$ under MLP-based semantics
is defined by 
$$\strengthMLP(a) =\begin{cases}
	\lim_{k \rightarrow \infty} s_a^{(k)} & \textrm{if the limit exists} \\
	\bot & \textrm{otherwise} 
\end{cases}
$$
for all $a \in \arguments$.
\end{definition}
In order to illustrate the definition, Figure \ref{fig:fig_QBAF_evaluated} shows the interpretation of our example
QBAF from Figure \ref{fig:fig_QBAF} for two different instantiations of edge weights.
\begin{figure}[tb]
	\centering
		\includegraphics[width=0.45\textwidth]{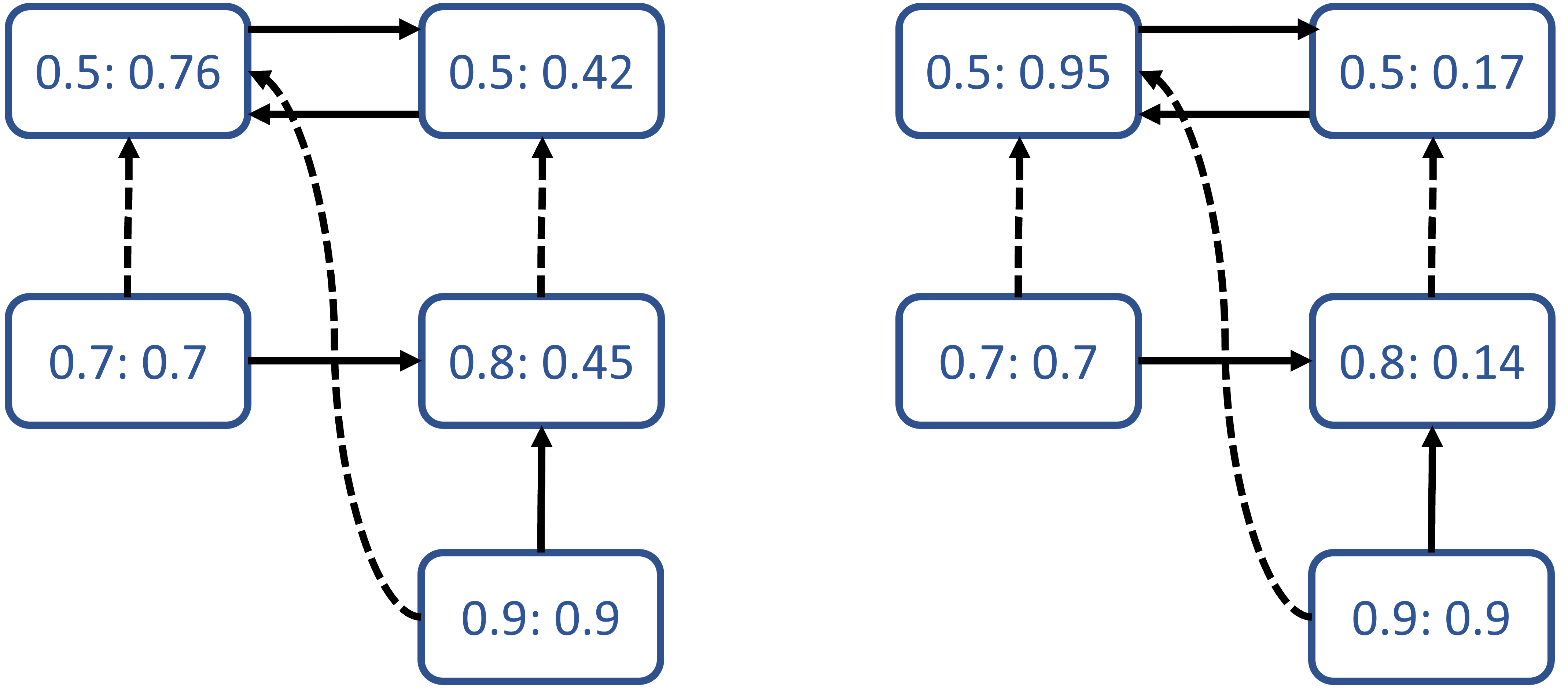}
	\caption{MLP-based interpretation of the QBAF from Figure \ref{fig:fig_QBAF}. The nodes are annotated with (base score: strength). The edge weights are $s$ for supports and $-s$ for attacks,
	where $s=1$ on the left and $s=2$ on the right. }
	\label{fig:fig_QBAF_evaluated}
\end{figure}

As we explain in the following proposition, if the MLP-based semantics is fully defined, then it corresponds to a fixed-point of
the update function $\discMLP$. This observation will be important later to study semantical properties.
\begin{proposition}
\label{prop_disc_limit_is_fixed_point}
If $\strengthMLP$ is fully defined, then $s^* = \lim_{k \rightarrow \infty} s^{(k)}$ is a fixed-point of 
$\discMLP$, i.e., $\discMLP(s^*) = s^*$.
\end{proposition}
\begin{proof}
See appendix.
\end{proof}
There are two main questions that we want to answer for a new modular semantics. 
The first question is, under which conditions does the iterative computation of strength values converge?
That is, for which families of QBAFs is the MLP-based semantics fully defined and are there families
for which it is not?
The second questions is, if the MLP-based semantics defines strength values, do they satisfy meaningful semantical
properties?
We will look at both questions in turn.

\subsection{Convergence Guarantees}

The following theorem explains some sufficient conditions under which the MLP-based semantics is fully defined.
The proofs build up on general results about modular semantics developed in \cite{potyka_modular_2019}.
\begin{theorem}
\label{theo_convergence_guarantees}
Let $Q $ be an edge-weighted QBAF over $[0,1]$.
\begin{enumerate}
	\item If $Q $ is acyclic, then $\strengthMLP$ is fully defined and, for all $a \in \arguments$, $\strengthMLP(a)$ can
		be computed in linear time.
	\item If all arguments in $Q$ have at most $P$ parents, the weight of all edges is bounded from above by $W$
	 and we have $W \cdot P < 4$, then $\strengthMLP$ is fully defined. 
	Furthermore, $|\strengthMLP(a) - s_a^{(n)}| < \epsilon$ whenever $n > \frac{\log \epsilon}{\log W + \log P - \log 4}$.
\end{enumerate}
\end{theorem}
\begin{proof}
See appendix.
\end{proof}
In the acyclic case in item 1, the strength values can basically be computed by a single forward pass over a topological ordering
of the arguments \cite{potyka_modular_2019}. It is interesting to note that this process is equivalent to the usual forward propagation process in feed-forward networks (because, in an MLP, every layerwise ordering from the input to the output layer corresponds to a topological ordering and vice versa). 
In this sense, MLPs can indeed be seen as special cases of QBAFs, where the QBAF
has an acyclic layered structure, the aggregation function is addition and the influence function is a neural network
activation function.

Item 2 explains more complicated convergence conditions for cyclic QBAFs and gives a guarantee for the convergence rate. 
Convergence can be guaranteed if the maximum number of parents $P$ of arguments and the maximum edge weight $W$ in the QBAF are not too large. 
For example, if all edge weights are strictly smaller than $W=1.3$ and every argument has at most $P=3$ parents, 
then the iterative procedure is guaranteed to converge and the interpretation is fully defined.
To understand the guarantees for the convergence rate, first note that $\log W + \log P - \log 4 = \log \frac{W \cdot P}{4} < \log(1) = 0$
by the assumption $W \cdot P < 4$. Hence, the denominator in the term $\frac{\log \epsilon}{\log W + \log P - \log 4}$ is always negative.
For $\epsilon > 1$, the fraction is negative and, in this case, the bound is trivially true because all strength values are between $0$ and $1$.
Indeed, we are usually interested in small values of $\epsilon$ close to $0$. In this case, both the numerator and denominator are negative.
In particular, $\log \epsilon \rightarrow -\infty$ as $\epsilon \rightarrow 0$. That is, the number of iterations $n$ needed until
the difference between $s_a^{(n)}$ and $\strengthMLP(a)$ is smaller than a desired accuracy $\epsilon$ grows with increasing accuracy as we 
would naturally expect.
Perhaps more surprising, the number of iterations decreases as $W$ and $P$ become larger. 
An intuitive explanation is that large weights and many parents will move the weights quicker such that
convergence occurs faster. Of course, large $W$ and $P$ can also cause divergence of the procedure, but this can only happen if 
$W \cdot P \geq 4$.

The conditions in Theorem \ref{theo_convergence_guarantees} are sufficient, but not necessary for convergence. However,
Figure \ref{fig:fig_QBAF_divergence} shows a QBAF that demonstrates that the guarantees cannot be improved significantly
without adding additional assumptions about the structure of the QBAF. 
\begin{figure}[tb]
	\centering
		\includegraphics[width=0.45\textwidth]{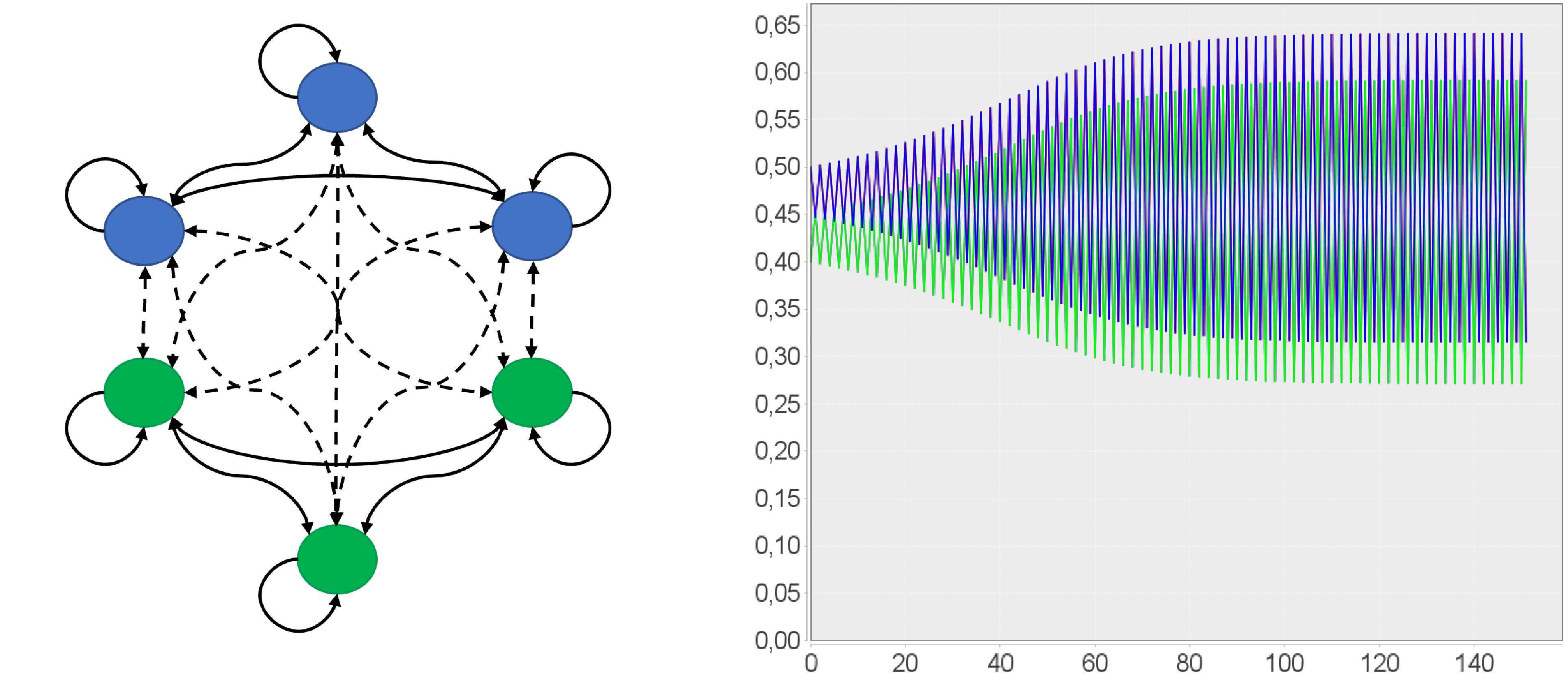}
	\caption{Left: Divergence example with base score $0.5$ (blue arguments) and $0.4$ (green arguments), edge weights $0.7$ (supports) and $-0.7$ (attacks).
	Right: evolution of strength values (y-axis) for blue and green arguments plotted against number of iterations (x-axis).}
	\label{fig:fig_QBAF_divergence}
\end{figure}
The QBAF in Figure \ref{fig:fig_QBAF_divergence} belongs to a family of QBAFs that have been presented 
in \cite{mossakowski2018modular} to construct divergence examples for modular semantics. 
Every blue argument attacks every blue argument (including itself) and supports every
green argument. Symmetrically, every green argument attacks every green argument and supports every blue argument.
The graph on the right in Figure \ref{fig:fig_QBAF_divergence} shows how the strength values evolve over time for green and
blue arguments. After approximately $100$ iterations, the strength values start cycling between two states.
Note that we have $W \cdot P =0.7 \cdot 6 = 4.2$. The example therefore shows that the condition $W \cdot P < 4$ in 
Theorem \ref{theo_convergence_guarantees} cannot be relaxed significantly. The example can be found in the 
Java library Attractor\footnote{\url{https://sourceforge.net/projects/attractorproject/}} \cite{Potyka2018_tut} in the folder examples/divergence. 
The reader
can check that the example still diverges for $W = 0.67$ ($W \cdot P = 4.02$). We present the example for $W=0.7$
mainly because the cycling can easily be illustrated visually for this case.

An overview of convergence guarantees for other modular semantics can be found in \cite{potyka_modular_2019}.
The convergence guarantees for MLP-based semantics are similarly strong as the ones for Euler-based semantics \cite{amgoud2017evaluation},
which are slightly stronger than the guarantees for DF-QuAD \cite{rago2016discontinuity} and the Quadratic Energy Model \cite{potyka2018Kr}.
While \cite{mossakowski2018modular} presented a modular semantics that guarantees convergence in general QBAFs, these guarantees
are bought at the expense of open-mindedness \cite{potyka_modular_2019}. That is, the strength values of arguments cannot be far from 
their original base scores. There is indeed a tradeoff between convergence guarantees and open-mindedness \cite{potyka_modular_2019}
and from this perspective, the MLP-based semantics is quite well behaved. 
Before we start discussing semantical guarantees of  MLP-based semantics, we take a detour in order to improve the convergence 
guarantees.

\subsection{Continuous MLP-Based Semantics}

As discussed in \cite{potyka2018Kr}, it is often possible to overcome convergence problems of modular semantics 
by continuizing their discrete update procedures. To do so, the update function of the modular semantics can be transformed 
into a system of differential equations.
\begin{definition}[Continuous MLP-based Semantics]
\label{def_cont_MLP_sem}
Let $Q $ be an edge-weighted QBAF over $[0,1]$. The interpretation of $Q$ under Continuous MLP-based Semantics
is defined by 
$$\strengthcMLP(a) =\begin{cases}
	\lim_{t \rightarrow \infty} \contMLP_a(t) & \textrm{if the limit exists} \\
	\bot & \textrm{otherwise} 
\end{cases}
$$
for all $a \in \arguments$, where $\contMLP: \mathbb{R}^+_0 \rightarrow [0,1]^n$ is the unique solution of the 
system of differential equations
\begin{align}
\label{eq_continuized_odes}
&\diff{f_i} = \frac{1}{1 + \frac{1-\baseScore(i)}{\baseScore(i)} \exp(-\sum_{(b,i) \in E} w(b,i) \cdot f_b)} - f_i, \\
&\ i=1,\dots,n, \notag
\end{align}
with initial conditions 
$f_i(0) = \baseScore(i)$ for $i = 1,\dots, n$. 
\end{definition}
Conceptually, the interpretation $\strengthcMLP$ is defined by two steps. First, we have to find the solution
$\contMLP$ of the system of differential equations \eqref{eq_continuized_odes}. Then we have to compute the limit of $\contMLP(t)$
as $t$ goes to infinity.
Intuitively, $\contMLP_a(t)$ can be understood as the strength of argument $a$ at time $t$. By the initial condition,
we have $\contMLP_a(0) = s^{(0)}_a = \baseScore(a)$, that is, the strength at time $0$ corresponds to the base score. 
As time progresses, the strength of $a$ continuously evolves.
In practice, the solution $\contMLP$ is approximated numerically and the two steps can be combined into one. The 
Java library Attractor \cite{Potyka2018_tut} contains an implementation of the Runge-Kutta method
RK4 for this purpose.

Intuitively, the i-th partial derivative $\diff{f_i}$ described in \eqref{eq_continuized_odes} describes the rate of change at a point in time and corresponds to the difference
between the desired function value \eqref{eq_discrete_update_formula} and the actual function value $f_i$. In particular, if $f_i$ 
is too large, the derivative will be negative so that the function value will decrease. Symmetrically, it will increase if $f_i$ is too
small. The following theorem explains that $\contMLP$ is indeed uniquely defined by the system of differential equations \eqref{eq_continuized_odes}
and explains some relationships between the discrete and continuous MLP-based semantics.
The proofs build up on general results about modular semantics developed in \cite{potyka_modular_2019}.
\begin{theorem}
\label{theorem_continuous_guarantees} 
For every QBAF $Q$, we have that
\begin{enumerate}
\item the system of differential equations in Definition \ref{def_cont_MLP_sem} has a unique solution
$\contMLP$.
\item If the limit $s^* = \lim_{t \rightarrow \infty} \contMLP(t)$ exists, then
$s^*$ is a fixed-point of $\discMLP$, that is, $\discMLP(s^*) = s^*$.
\item  If $\lim_{t \rightarrow \infty} \contMLP(t)$ converges and $Q$ satisfies any of the convergence conditions 
from Theorem \ref{theo_convergence_guarantees}, then $\strengthcMLP = \strengthMLP$.
\end{enumerate}
\end{theorem}
\begin{proof}
See appendix.
\end{proof}
Item 2 explains that whenever the continuous MLP-based semantics defines strength values, these strength values correspond 
to a fixed-point of the discrete update function. Note that the same is true for the discrete semantics as explained
in Proposition \ref{prop_disc_limit_is_fixed_point}. Unfortunately, it is not obvious that the fixed-points are equal
because $\discMLP$ may have several fixed-points.
However, item 3 explains that if the continuous MLP-based semantics defines strength values,
and any of the convergence conditions from Theorem \ref{theo_convergence_guarantees} are met,
then the fixed-points and thus the strength values are equal. 
Note that this applies, in particular, to acyclic graphs and graphs with small indegree or small weights. What makes this relationship
particularly interesting is that the continuous model can still converge to a meaningful limit when the discrete model does not.
Since this limit is guaranteed to be a fixed-point of the discrete model, it is, in a way, consistent with the discrete semantics.

Figure \ref{fig:fig_continuized} shows on the left how the strength values under continuous MLP-based semantics evolve
for the QBAF from Figure \ref{fig:fig_QBAF_divergence}. As opposed to the iterative update procedure, the continuous 
update process changes the strength values continuously and does indeed converge.  
This example demonstrates that the continuous model offers strictly stronger convergence 
guarantees than the discrete one. 
The intuitive reason is that every discrete modular semantics with smooth aggregation and influence function 
can be seen as a coarse approximation of a continuous counterpart \cite{potyka2018Kr}. From this perspective, the 
convergence problems for discrete semantics occur because the step-size of the approximation is too large.
It is actually an open question if there are QBAFs for which continuized semantics diverge as well. Until now, neither divergence examples
nor general convergence proofs have been found. To illustrate the general relationship between discrete and continuous MLP-based semantics further,
Figure \ref{fig:fig_continuized} shows, on the right, the evolution of strength values under discrete and continuous semantics
for the QBAF from Figure \ref{fig:fig_QBAF}. 
\begin{figure}[tb]
	\centering
		\includegraphics[width=0.45\textwidth]{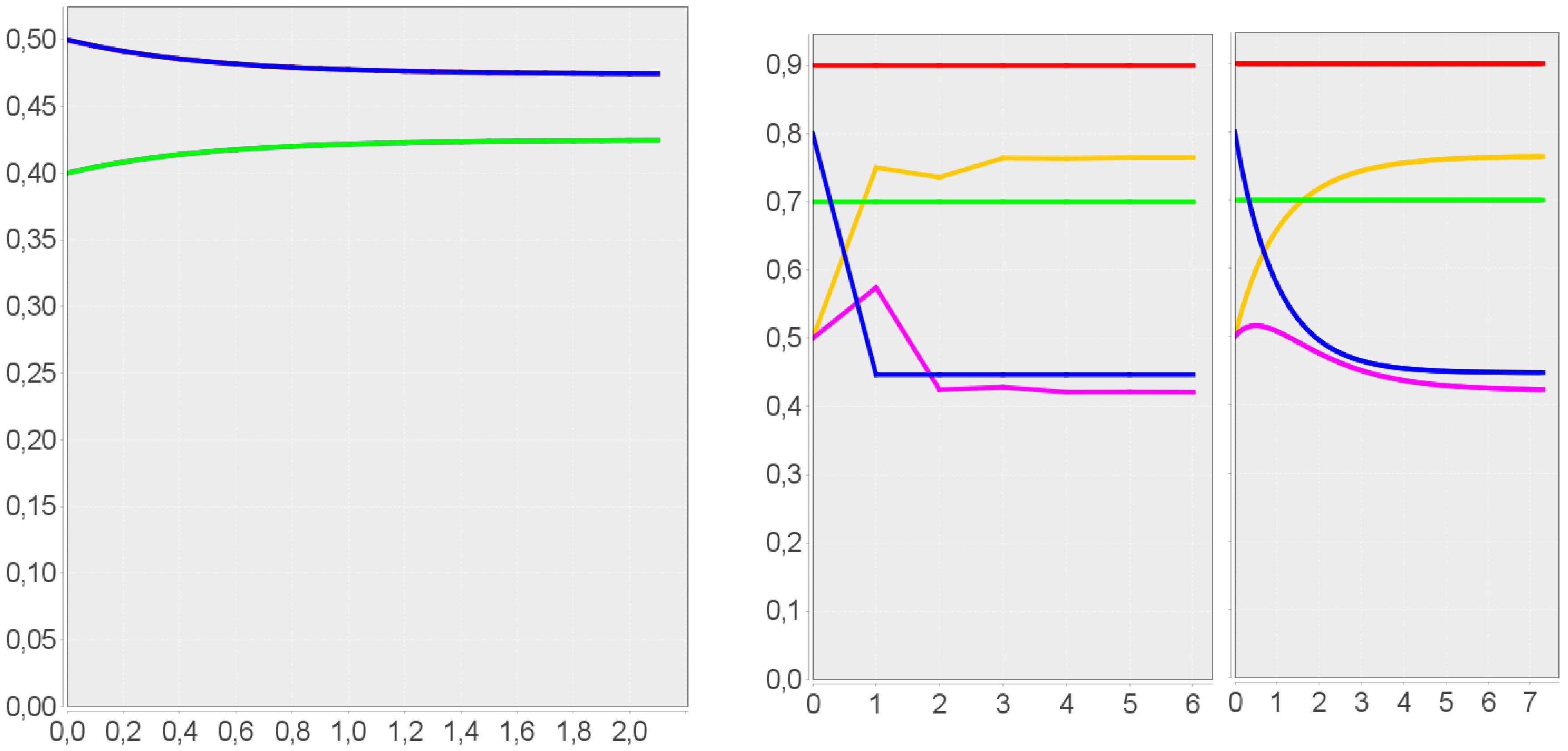}
	\caption{Evolution of strength values (y-axis) for QBAF from 
Figure \ref{fig:fig_QBAF_divergence} under continuous MLP-based semantics (left) and comparison of evolution of strength values
for QBAF from Figure \ref{fig:fig_QBAF} with $s=1$ under discrete and continuous  MLP-based semantics (right)}
	\label{fig:fig_continuized}
\end{figure}

\subsection{Semantical Guarantees}

We will now look at semantical guarantees for MLP-based semantics.
We know from Proposition \ref{prop_disc_limit_is_fixed_point} and Theorem  \ref{theorem_continuous_guarantees} that the strength values
under both semantics correspond to fixed-points of $\discMLP$ if they are defined. Therefore, we can study the properties of
both semantics simultaneously by studying properties that hold in a fixed-point of $\discMLP$.
In \cite{amgoud2017evaluation}, $12$ desirable properties have been presented that should be satisfied by quantitative argumentation
semantics. We consider two additional properties from \cite{potyka2018Kr,Potyka19om} that have been motivated by shortcomings
of existing semantics.
Since the properties have been phrased for QBAFs without edge-weights, we assume that the weights of all supports are $1$
and the weights of all attacks are $-1$.
To phrase the  properties, we let $\attacker^+$ and $\supporter^+$ denote the subsets of arguments 
in $\attacker$ and $\supporter$ that the fixed-point assigns a non-zero strength to.
The last property \emph{Almost Open-Mindedness} is a slightly weaker form of \emph{Open-Mindedness} from \cite{Potyka19om}. The only
difference to the original definition is that it excludes the base scores $0$ and $1$. 
\begin{theorem}
\label{theorem_properties_mlp_semantics}
Consider edge-weighted QBAFs $Q = (\arguments, E, \baseScore, w)$ and $Q' = (\arguments', E', \baseScore', w')$ with $w(e), w'(e') \in \{-1,1\}$ for all $e \in E, e' \in E'$ and corresponding
interpretations $\sigma$ and $\sigma'$ under discrete or continuous MLP-based semantics. Then the following properties are satisfied:
\begin{description}
	\item[Anonymity:] If $Q$ and $Q'$ are ismomorphic, then $\sigma = \sigma'$. 
	\item[Independence:] If $\arguments \cap \arguments' = \emptyset$, 
	  then for $Q'' = (\arguments \cup \arguments',  E \cup E', \baseScore \cup \baseScore', w \cup w')$,
		$\sigma''$ is fully defined, 
		$\sigma''(a) = \sigma(a)$ for $a \in \arguments$ and $\sigma''(a) = \sigma'(a)$ for $a \in \arguments'$.
 \item[Directionality:] If $\arguments = \arguments'$ and $E = E' \cup \{(a, b)\}$,
  then for all $c \in \arguments$ such that there is no directed path from $b$ to $c$, we have $\sigma(c) = \sigma'(c)$.
 \item[Equivalence:] If there are $a, b \in \arguments$ such that $\baseScore(a) = \baseScore(b)$ and there are bijections $h: \attacker(a) \rightarrow \attacker(b)$, $h': \supporter(a) \rightarrow \supporter(b)$
such that $\sigma(x) = \sigma(h(x))$ and $\sigma(y) = \sigma(h'(y))$ for all $x \in \attacker(a), y \in \supporter(a)$, then $\sigma(a) = \sigma(b)$.
 \item[Stability:] If there is an $a \in \arguments$ such that 
  $\attacker(a) = \supporter(a) = \emptyset$, then $\sigma(a) = \baseScore(a)$.
 \item[Neutrality:] If there are $a, b \in \arguments$ such that $\baseScore(a) = \baseScore(b)$, $\attacker(a) \subseteq \attacker(b)$, $\supporter(a) \subseteq \supporter(b)$,
$\attacker(a) \cup \supporter(a) = \attacker(b) \cup \supporter(b) \cup \{d\}$ and $\sigma(d) = 0$, then $\sigma(a) = \sigma(b)$.  
 \item[Monotony:] If there are $a, b \in \arguments$ such that $0 <\baseScore(a) = \baseScore(b) < 1$, $\attacker(a) \subseteq \attacker(b)$, $\supporter(a) \supseteq \supporter(b)$,
then
\begin{enumerate}
	\item $\sigma(a) \geq \sigma(b)$. \hfill (Monotony)
	\item if furthermore ($\sigma(a) > 0$ or $\sigma(b)<1$) and ($\attacker(a)^+ \subset \attacker(b)^+$ or $\supporter(a)^+ \supset \supporter(b)^+$), then $\sigma(a) >  \sigma(b)$. \hfill (Strict Monotony)
\end{enumerate} 
 \item[Reinforcement:] If there are $a, b \in \arguments$ such that $0 < \baseScore(a) = \baseScore(b) < 1$, $\attacker(a) \setminus \{x\} = \attacker(b)  \setminus \{y\}$, 
$\supporter(a)  \setminus \{x'\}  = \supporter(b)  \setminus \{y'\} $, $\sigma(x) \leq \sigma(y)$ and $\sigma(x') \geq \sigma(y')$, 
then
\begin{enumerate}
	\item $\sigma(a) \geq \sigma(b)$. \hfill (Reinforcement)
	\item if ($\sigma(a) > 0$ or $\sigma(b)<1$) and ($\sigma(x) < \sigma(y)$ or $\sigma(x') > \sigma(y')$), 
	then $\sigma(a) > \sigma(b)$. \hfill (Strict Reinforcement)
\end{enumerate} 
\item[Resilience:] If $a \in \arguments$ is such that $0 < \baseScore(a) < 1$, then $0 < \sigma(a) < 1$.
 \item[Franklin:] If there are 
  $a, b \in \arguments$ such that $\baseScore(a) = \baseScore(b)$, $\attacker(a) = \attacker(b) \cup \{x\}$, $\supporter(a) = \supporter(b)  \cup \{y\}$
and $\sigma(x) = \sigma(y)$, then $\sigma(a) = \sigma(b)$.
\item[Weakening: ] Assume that there is an $a \in \arguments$ with $\baseScore(a) > 0$. 
 Assume further that $g: \supporter(a) \rightarrow \attacker(a)$ is an injective function such that $\sigma(x) \leq \sigma(g(x))$
for all $x \in \supporter(a)$ and ($\attacker(a)^+ \setminus g(\supporter(a)) \neq \emptyset$ or there is an $x \in \supporter(a)$ such that $\sigma(x) < \sigma(g(x))$).
Then $\sigma(a) < \baseScore(a)$.
 \item[Strengthening:] Assume that there is an $a \in \arguments$ with $\baseScore(a) <1$. 
 Assume further that $b: \attacker(a) \rightarrow \supporter(a)$ is an injective function such that $\sigma(x) \leq \sigma(b(x))$
for all $x \in \attacker(a)$ and ($\supporter(a)^+ \setminus b(\attacker(a)) \neq \emptyset$ or there is an $x \in \attacker(a)$ such that $\sigma(x) < \sigma(b(x))$).
Then $\sigma(a) > \baseScore(a)$.
 \item[Duality:] Assume that there are $a, b \in \arguments$ such that 
  $\baseScore(a) = 0.5 + \epsilon$, $\baseScore(b) = 0.5 - \epsilon$ for some $\epsilon \in [0,0.5]$.
  If there are bijections $h: \attacker(a) \rightarrow \supporter(b)$, $h': \supporter(a) \rightarrow \attacker(b)$
such that $\sigma(x) = \sigma(f(x))$ and $\sigma(y) = \sigma(g(y))$ for all 
  $x \in \attacker(a), y \in \supporter(a)$, then $\sigma(a) - \baseScore(a) = \baseScore(b) - \sigma(b)$.
 \item[Almost Open-Mindedness:] For all $k \in \mathbb{N}$ and $p \in \{-1,1\}$, let $Q_k^p = (\arguments_k^p, E_k^p, \baseScore_k^p, w_k^p)$
   be constructed from $Q$ by letting $\arguments_k^p = \arguments \cup \{A_1, \dots, A_k\}$, 
		$E_k^p = E \cup \{(A_1, a), \dots, (A_k, a)\}$, $\baseScore_k^p(b) = \baseScore(b)$ for all $b \in \arguments$ and
		 $\baseScore_k^p(A_i) = p$ for $1 \leq i \leq k$.
		 Then for every $a \in \arguments$ with $0 < \baseScore(a) < 1$ and for every $\epsilon > 0$, 
    there is an $N \in \mathbb{N}$ such that the interpretation $\sigma^{k,p}$ corresponding to $Q_k^p$ satisfies
		\begin{enumerate}
			\item $\sigma^{k,p}(a) < \epsilon$ whenever $p=-1$ and $k > N$ and
			\item $\sigma^{k,p}(a) > 1 - \epsilon$ whenever $p=1$ and $k > N$.
		\end{enumerate}
\end{description}
\end{theorem}
\begin{proof}
See appendix.
\end{proof}
\begin{figure}[tb]
	\centering
		\includegraphics[width=0.45\textwidth]{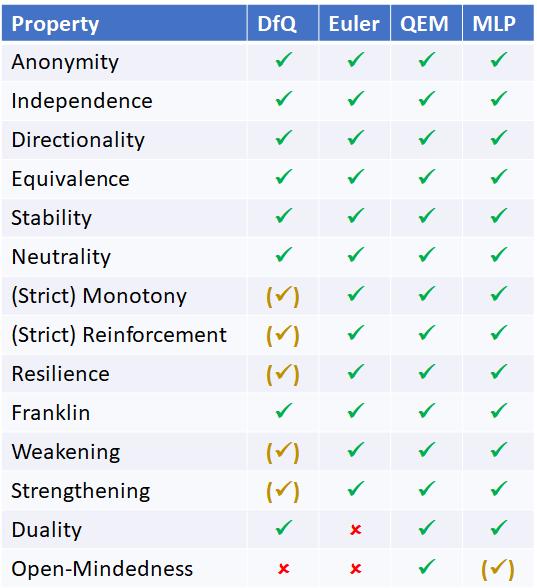}
	\caption{Semantical properties that are satisfied ($\checkmark$), satisfied when excluding base scores $0$ and $1$ ($(\checkmark)$) or not satisfied even when 
	excluding base scores $0$ or $1$ ($\textit{x}$) by 
	Df-QuAD (DfQ), Euler-based Semantics (Euler), Quadratic Energy Model 
 (QEM)	and MLP-based Semantics (MLP).}
	\label{fig:fig_properties}
\end{figure}
The first 12 properties have been introduced in \cite{amgoud2017evaluation}.
\emph{Anonymity} is a fairness condition and intuitively states that the  strength values should not depend on the identity of the argument.
\emph{Independence} says that disconnected subgraphs should not affect each other. 
\emph{Directionality} demands that the strength of an argument depends only on its predecessors in the graph.
\emph{Equivalence} says that arguments with equal status should be evaluated equally.
\emph{Stability} states that the final strength is just the initial weight if an argument does not have any parents.
\emph{Neutrality} demands that arguments with strength $0$ do not affect other arguments.
\emph{Monotony} makes a quantitative statement: adding attackers or removing supporters can only weaken an argument.
\emph{Reinforcement} makes a similiar qualitative statement: strengthening attackers or weakening supporters can only weaken an argument.
\emph{Resilience} demands that the extreme values $0$ and $1$ can never be taken unless the base score was already an extreme value.
\emph{Franklin} says that an attacker and a supporter with equal strength cancel their effects. 
\emph{Weakening} states that an argument's strength must be smaller than its base score when the attackers dominate
the supporters. Symmetrically, Strengthening says that its strength must be larger when the supporters dominate.
\emph{Duality} from \cite{potyka2018Kr} demands that attacks and supports are treated equally. Roughly speaking,
the positive effect of a support should correspond to the negative effect of an attack.
\emph{Open-mindedness} \cite{Potyka19om} says that the strength of an argument can become arbitrarily close to $0$ or $1$
independent of its base score if there is only a sufficient number of strong attackers or supporters. 
As we explain in the appendix, the MLP-based semantics satisfies this property in almost all cases except if 
base scores are set to $0$ and $1$. In this case, they can actually never change under MLP-based semantics.

Figure \ref{fig:fig_properties} gives an overview about which properties are satisfied by different semantics.
Df-QuAD \cite{rago2016discontinuity} had been introduced first and already fixed a problem of the
QuAD model proposed in \cite{baroni2015automatic}. However, it does not completely satisfy several properties because of the
way how it aggregates strength values. Roughly speaking, if an argument has both an attacker and a supporter 
with strength $1$, its strength will necessarily be the base score no matter what other attackers and
supporters there are. The Euler-based semantics \cite{amgoud2017evaluation} had been introduced to overcome these problems. 
However, it introduced some other problems that are reflected by the fact that it satisfies neither duality nor open-mindedness. 
In particular, it treats attacks and supports in a rather random asymmetrical fashion. 
The quadratic energy model \cite{potyka2018Kr} had been introduced to fix these issues. Therefore, it is not surprising that it satisfies
all properties.
Perhaps more surprising is that the MLP-based semantics satisfies all properties almost perfectly even though
it has not been designed for this purpose. Its mechanics are actually very similar to the Euler-based semantics,
but it fixes the Euler-based semantics' asymmetry between attacks and supports.
 As we explain in the appendix, the MLP-based semantics violates \emph{Open-Mindedness} only when the base
scores are set to the extreme values $0$ or $1$. It is a little bit odd that these values cannot change since they 
basically render such arguments redundant (their effect could directly be encoded in the base score of their children).
However, it is not a big drawback since there is usually not a big practical difference between the base 
scores $0.99$ and $1$ or $0.01$ and $0$, respectively.

\section{Conclusions and Related Work}

We viewed MLPs as QBAFs to analyze their mechanics from an argumentation perspective. As it turns out, the MLP-based semantics 
offers comparatively good convergence guarantees in cyclic QBAFs and satisfies the common-sense properties from the literature almost perfectly. 
Recent combinations of machine learning methods and QBAFs often use variants of Df-QuAD and Euler-based semantics \cite{cocarascu2019extracting,kotonya2019gradual}. It may be interesting to evaluate these approaches with MLP-based semantics. In particular, the generated QBAFs are acyclic in many applications, 
so that the resulting model under MLP-based semantics is a sparse MLP. For applications, this is interesting because it allows to retrain the 
weights by the usual backpropagation procedure in an end-to-end fashion (base score $\beta$ translates to bias $\ln(\beta/(1- \beta))$ and bias $\theta$ translates to base score $\varphi_l(\theta)$). 
From a machine learning perspective, this is interesting because there has been growing interest in learning sparse neural networks 
\cite{louizos2018learning,frankle2018lottery,mocanu2018scalable}, not only to improve their interpretability,
but also to tame their learning complexity.  We may create sparse MLPs
by building an acyclic sparse QBAFs from data like in \cite{cocarascu2019extracting,kotonya2019gradual}
and translating it into an MLP.

It seems, more generally, interesting to view an acyclic QBAF with sum for aggregation as an MLP with a particular activation 
function to learn base scores and edge weights of QBAFs from data.   
If the influence function is differentiable, we can indeed just use the usual backpropagation procedure that is implemented in libraries like PyTorch and Tensorflow.

Let us note that there has been previous work on using neural networks for argumentation. 
For example, the authors in \cite{garcez2005value} showed how \emph{value-based argumentation frameworks} \cite{bench2003persuasion}
can be encoded as MLPs.
In these frameworks, every argument is associated with a \emph{value} and there is a set of audiences with
different preferences over the values. Arguments can then be \emph{subjectively accepted} by one or
\emph{objectively accepted} by all audiences. The authors in \cite{garcez2005value} showed that an MLP with a single hidden layer
and a \emph{semi-linear activation function} can compute the \emph{prevailing arguments} in these frameworks.
More recently, there have also been attempts to use neural networks to approximately compute labellings of classical argumentation
frameworks \cite{riveret2015neuro,kuhlmann2019using}.

Argumentation technology has also been considered as a more immediate tool for interpretable machine learning.
\cite{thimm2017towards} proposed to solve classification problems by means
of \emph{structured argumentation}. As opposed to the abstract argumentation setting that we considered here, 
structured argumentation explicitly takes the premises and conclusions of arguments into account. 
\cite{thimm2017towards} suggest learning
structured arguments by rule mining algorithms. The rules 
can then be fed into a structured argumentation solver that can then derive a label for given inputs
and explain the outcome. 
While this is a very interesting idea for explainable classification, a current challenge is guiding the underlying rule mining algorithm
such that it finds meaningful arguments. 

\subsection{Acknowledgements:} This research was supported by the DFG through the projects EVOWIPE (STA572/15-1)
and COFFEE (STA572/15-2).

\section*{Appendix}

\newtheorem*{prop_disc_limit_is_fixed_point}{Proposition \ref{prop_disc_limit_is_fixed_point}}
\begin{prop_disc_limit_is_fixed_point}
If $\strengthMLP$ is fully defined, then $s^* = \lim_{k \rightarrow \infty} s^{(k)}$ is a fixed-point of 
$\discMLP$, i.e., $\discMLP(s^*) = s^*$.
\end{prop_disc_limit_is_fixed_point}
\begin{proof}
Note first that the update function $\discMLP$ is a continuous function on $[0,1]$ by our definition. 
Hence, we have  $\discMLP(s^*) = \discMLP(\lim_{k \rightarrow \infty} s^{(k)})
=\discMLP(\lim_{k \rightarrow \infty} \discMLP^k(s^{(0)})) = \lim_{k \rightarrow \infty} \discMLP^{k+1}(s^{(0)})
= s^*$, where the third equality follows from continuity of $\discMLP$.
\end{proof}

Theorems 1 and 2 follow from observing that the MLP-based semantics belongs to the class of \emph{Basic Modular Semantics}
that were introduced in \cite{potyka_modular_2019}.
We explain this in the following Lemma.
\begin{lemma}
The MLP-based semantics is a basic modular semantics.
\end{lemma}
\begin{proof}
To prove the claim, we have to check that the aggregation and influence function satisfy the properties of basic modular semantics
stated in Definition 2.4 in \cite{potyka_modular_2019}. The aggregation function is just a weighted variant of the sum aggregation
function considered in \cite{potyka_modular_2019} and the proofs are completely analogously to the corresponding proofs
for Proposition 2.5 in \cite{potyka_modular_2019} (see \url{https://arxiv.org/pdf/1809.07133.pdf} for the proofs).

The influence function must satisfy two properties. First, it must return the base score of an argument whenever the aggregate is $0$. 
To see that this is the case, note that 
\begin{align*}
\varphi_l\big(\ln(\frac{\baseScore(a)}{1- \baseScore(a)}) + 0 \big)
&= \frac{1}{1 + \exp(-\ln(\frac{\baseScore(a)}{1- \baseScore(a)}))}\\
&= \frac{1}{1 + \frac{1- \baseScore(a)}{\baseScore(a)}} 
= \baseScore(a).
\end{align*}
The second property that the influence function must satisfy is Lipschitz-continuity \cite{rudin1976}. To see that it does, first note that the influence function is a
function of the aggregate (the base score is a constant). 
We make use of the fact that a function with derivative bounded by $B$ is Lipschitz-continuous with Lipschitz constant $B$ (this can be seen from the mean value theorem \cite{rudin1976}). It is well known that the derivate of the logistic function is
$\varphi'_l(z) = \varphi_l(z) \cdot \varphi_l(-z)$. It takes its maximum at $0$ and is therefore bounded by $\varphi_l(0) \cdot \varphi_l(0) = 0.5 \cdot 0.5 = \frac{1}{4}$. The mean value
theorem therefore implies that it is Lipschitz-continuous with Lipschitz constant $\frac{1}{4}$.   
\end{proof}

\newtheorem*{theo_convergence_guarantees}{Theorem \ref{theo_convergence_guarantees}}
\begin{theo_convergence_guarantees}
Let $Q $ be an edge-weighted QBAF over $[0,1]$.
\begin{enumerate}
	\item If $Q $ is acyclic, then $\strengthMLP$ is fully defined and, for all $a \in \arguments$, $\strengthMLP(a)$ can
		be computed in linear time.
	\item If all arguments in $Q$ have at most $P$ parents, the weight of all edges is bounded from above by $W$
	 and we have $W \cdot P < 4$, then $\strengthMLP$ is fully defined. 
	Furthermore, $|\strengthMLP(a) - s_a^{(n)}| < \epsilon$ whenever $n > \frac{\log \epsilon}{\log W + \log P - \log 4}$.
\end{enumerate}
\end{theo_convergence_guarantees}
\begin{proof}
Item 1 follows from Lemma 1 and Proposition 3.1 in \cite{potyka_modular_2019}.

For Item 2, note that $\lambda^\alpha_i = \sum_{(b,i) \in E} w(b,i)$ is a Lipschitz constant for the aggregation function (weighted sum) 
at the $i$-th component. In particular, $\lambda_i \leq W \cdot P$. As explained in the proof of Lemma 1, $\lambda^\iota_i = \frac{1}{4}$
is a Lipschitz constant for the influence function, so that $\lambda^\alpha_i \cdot \lambda^\iota_i \leq \frac{W \cdot P}{4}$. Every component
of the update function is therefore Lipschitz-continuous with Lipschitz constant $\frac{W \cdot P}{4}$.
Item 2 follows from this with Proposition 3.3 in \cite{potyka_modular_2019}.
\end{proof}

\newtheorem*{theorem_continuous_guarantees}{Theorem \ref{theorem_continuous_guarantees}}
\begin{theorem_continuous_guarantees}
For every BAG $Q$, we have that
\begin{enumerate}
\item the system of differential equations in Definition 6 has a unique solution
$\contMLP$.
\item If the limit $s^* = \lim_{t \rightarrow \infty} \contMLP(t)$ exists, then
$s^*$ is a fixed-point of $\discMLP$, that is, $\discMLP(s^*) = s^*$.
\item  If $\lim_{t \rightarrow \infty} \contMLP(t)$ converges and $Q$ satisfies any of the convergence conditions 
from Theorem \ref{theo_convergence_guarantees}, then $\strengthcMLP = \strengthMLP$.
\end{enumerate}
\end{theorem_continuous_guarantees}
\begin{proof}
All claims follow from Lemma 1 and Proposition 4.1 in \cite{potyka_modular_2019}.
\end{proof}

To phrase the semantical properties, we let $\attacker^+$ and $\supporter^+$ denote the subsets of arguments 
in $\attacker$ and $\supporter$ that the (discrete or continuous) MLP-based semantics assigns a non-zero strength to.
The last property \emph{Almost Open-Mindedness} is a relaxation of \emph{Open-Mindedness} \cite{Potyka19om}. The only
difference to the original definition is that it excludes the base scores $0$ and $1$. 
\newtheorem*{theorem_properties_mlp_semantics}{Theorem \ref{theorem_properties_mlp_semantics}}
\begin{theorem_properties_mlp_semantics}
Let $Q = (\arguments, E, \baseScore, w)$ and $Q' = (\arguments', E', \baseScore', w')$ be edge-weighted QBAFs with $w(e), w'(e') \in \{-1,1\}$ for all $e \in E, e' \in E'$ and corresponding
interpretations $\sigma$ and $\sigma'$ under discrete or continuous MLP-based semantics. Then the following properties are satisfied:
\begin{description}
	\item[Anonymity:] If $Q$ and $Q'$ are ismomorphic, then $\sigma = \sigma'$. 
	\item[Independence:] If $\arguments \cap \arguments' = \emptyset$, 
	  then for $Q'' = (\arguments \cup \arguments',  E \cup E', \baseScore \cup \baseScore', w \cup w')$,
		$\sigma''$ is fully defined, 
		$\sigma''(a) = \sigma(a)$ for $a \in \arguments$ and $\sigma''(a) = \sigma'(a)$ for $a \in \arguments'$.
 \item[Directionality:] If $\arguments = \arguments'$ and $E = E' \cup \{(a, b)\}$,
  then for all $c \in \arguments$ such that there is no directed path from $b$ to $c$, we have $\sigma(c) = \sigma'(c)$.
 \item[Equivalence:] If there are $a, b \in \arguments$ such that $\baseScore(a) = \baseScore(b)$ and there are bijections $h: \attacker(a) \rightarrow \attacker(b)$, $h': \supporter(a) \rightarrow \supporter(b)$
such that $\sigma(x) = \sigma(h(x))$ and $\sigma(y) = \sigma(h'(y))$ for all $x \in \attacker(a), y \in \supporter(a)$, then $\sigma(a) = \sigma(b)$.
 \item[Stability:] If there is an $a \in \arguments$ such that 
  $\attacker(a) = \supporter(a) = \emptyset$, then $\sigma(a) = \baseScore(a)$.
 \item[Neutrality:] If there are $a, b \in \arguments$ such that $\baseScore(a) = \baseScore(b)$, $\attacker(a) \subseteq \attacker(b)$, $\supporter(a) \subseteq \supporter(b)$,
$\attacker(a) \cup \supporter(a) = \attacker(b) \cup \supporter(b) \cup \{d\}$ and $\sigma(d) = 0$, then $\sigma(a) = \sigma(b)$.  
 \item[Monotony:] If there are $a, b \in \arguments$ such that $0 <\baseScore(a) = \baseScore(b) < 1$, $\attacker(a) \subseteq \attacker(b)$, $\supporter(a) \supseteq \supporter(b)$,
then
\begin{enumerate}
	\item $\sigma(a) \geq \sigma(b)$. \hfill (Monotony)
	\item if furthermore ($\sigma(a) > 0$ or $\sigma(b)<1$) and ($\attacker(a)^+ \subset \attacker(b)^+$ or $\supporter(a)^+ \supset \supporter(b)^+$), then $\sigma(a) >  \sigma(b)$. \hfill (Strict Monotony)
\end{enumerate} 
 \item[Reinforcement:] If there are $a, b \in \arguments$ such that $0 < \baseScore(a) = \baseScore(b) < 1$, $\attacker(a) \setminus \{x\} = \attacker(b)  \setminus \{y\}$, 
$\supporter(a)  \setminus \{x'\}  = \supporter(b)  \setminus \{y'\} $, $\sigma(x) \leq \sigma(y)$ and $\sigma(x') \geq \sigma(y')$, 
then
\begin{enumerate}
	\item $\sigma(a) \geq \sigma(b)$. \hfill (Reinforcement)
	\item if ($\sigma(a) > 0$ or $\sigma(b)<1$) and ($\sigma(x) < \sigma(y)$ or $\sigma(x') > \sigma(y')$), 
	then $\sigma(a) > \sigma(b)$. \hfill (Strict Reinforcement)
\end{enumerate} 
\item[Resilience:] If $a \in \arguments$ is such that $0 < \baseScore(a) < 1$, then $0 < \sigma(a) < 1$.
 \item[Franklin:] If there are 
  $a, b \in \arguments$ such that $\baseScore(a) = \baseScore(b)$, $\attacker(a) = \attacker(b) \cup \{x\}$, $\supporter(a) = \supporter(b)  \cup \{y\}$
and $\sigma(x) = \sigma(y)$, then $\sigma(a) = \sigma(b)$.
\item[Weakening: ] Assume that there is an $a \in \arguments$ with $\baseScore(a) > 0$. 
 Assume further that $g: \supporter(a) \rightarrow \attacker(a)$ is an injective function such that $\sigma(x) \leq \sigma(g(x))$
for all $x \in \supporter(a)$ and ($\attacker(a)^+ \setminus g(\supporter(a)) \neq \emptyset$ or there is an $x \in \supporter(a)$ such that $\sigma(x) < \sigma(g(x))$).
Then $\sigma(a) < \baseScore(a)$.
 \item[Strengthening:] Assume that there is an $a \in \arguments$ with $\baseScore(a) <1$. 
 Assume further that $b: \attacker(a) \rightarrow \supporter(a)$ is an injective function such that $\sigma(x) \leq \sigma(b(x))$
for all $x \in \attacker(a)$ and ($\supporter(a)^+ \setminus b(\attacker(a)) \neq \emptyset$ or there is an $x \in \attacker(a)$ such that $\sigma(x) < \sigma(b(x))$).
Then $\sigma(a) > \baseScore(a)$.
 \item[Duality:] Assume that there are $a, b \in \arguments$ such that 
  $\baseScore(a) = 0.5 + \epsilon$, $\baseScore(b) = 0.5 - \epsilon$ for some $\epsilon \in [0,0.5]$.
  If there are bijections $h: \attacker(a) \rightarrow \supporter(b)$, $h': \supporter(a) \rightarrow \attacker(b)$
such that $\sigma(x) = \sigma(f(x))$ and $\sigma(y) = \sigma(g(y))$ for all 
  $x \in \attacker(a), y \in \supporter(a)$, then $\sigma(a) - \baseScore(a) = \baseScore(b) - \sigma(b)$.
 \item[Almost Open-Mindedness:] For all $k \in \mathbb{N}$ and $p \in \{-1,1\}$, let $Q_k^p = (\arguments_k^p, E_k^p, \baseScore_k^p, w_k^p)$
   be constructed from $Q$ by letting $\arguments_k^p = \arguments \cup \{A_1, \dots, A_k\}$, 
		$E_k^p = E \cup \{(A_1, a), \dots, (A_k, a)\}$, $\baseScore_k^p(b) = \baseScore(b)$ for all $b \in \arguments$ and
		 $\baseScore_k^p(A_i) = p$ for $1 \leq i \leq k$.
		 Then for every $a \in \arguments$ with $0 < \baseScore(a) < 1$ and for every $\epsilon > 0$, 
    there is an $N \in \mathbb{N}$ such that the interpretation $\sigma^{k,p}$ corresponding to $Q_k^p$ satisfies
		\begin{enumerate}
			\item $\sigma^{k,p}(a) < \epsilon$ whenever $p=-1$ and $k > N$ and
			\item $\sigma^{k,p}(a) > 1 - \epsilon$ whenever $p=1$ and $k > N$.
		\end{enumerate}
\end{description}
\end{theorem_properties_mlp_semantics}
\begin{proof}
Note first that since $\sigma$ corresponds to a fixed-point of $\discMLP$ and all edge weights are either $1$ or $-1$, 
we have, for all $a \in \arguments$, that $\sigma(a)$ equals
\begin{equation}
\label{eq_equilibrium}
\frac{1}{1 + \frac{1-\baseScore(a)}{\baseScore(a)}\exp(\sum\limits_{b \in \attacker(a)}  \sigma(b) - \sum\limits_{b \in \supporter(a)}  \sigma(b))}.
\end{equation}

Anonymity follows immediately from observing that the strength of an argument depends only on its base score and the strength of its attackers
and supporters and does not depend on its identity. 

Independence follows immediately from the fact that the arguments in $Q$ and $Q'$ are completely independent. The result of the update function 
operating on $Q$ and $Q'$ simultaneously is therefore just the combination of the results of the update function individually operating on $Q$ and $Q'$,
respectively. Similarly, the solution of the combined system of differential equations for $Q$ and $Q'$ is just the combination of the individual solutions for 
$Q$ and $Q'$, respectively.

For Directionality, note from \eqref{eq_equilibrium} that the strength of every argument depends only on the strength of its parents. 
Since there is no path from $b$ to $c$, $b$ cannot be parent of any of $c$'s predecessors. Since the subgraph consisting of $c$'s parents
in $Q$ equals the corresponding subgraph in $Q'$, the interpretation of all arguments in this subgraph must be equal.

For Equivalence, we get from \eqref{eq_equilibrium} that 
$\sigma(a)  
= \frac{1}{1 + \frac{1-\baseScore(a)}{\baseScore(a)}\exp(\sum\limits_{x \in \attacker(a)}  \sigma(x) - \sum\limits_{x \in \supporter(a)}  \sigma(x))}
= \frac{1}{1 + \frac{1-\baseScore(b)}{\baseScore(b)}\exp(\sum\limits_{x \in \attacker(a)}  \sigma(x) - \sum\limits_{x \in \supporter(a)}  \sigma(x))}
 = \sigma(b)$, where we used the assumption that the base scores and the strength values of attackers and supporters of $a$ and $b$ are equal.

For Stability, we get from \eqref{eq_equilibrium} that 
$\sigma(a) = \frac{1}{1 + \frac{1-\baseScore(a)}{\baseScore(a)}\exp(0)}= \baseScore(a)$.

Neutrality follows again from \eqref{eq_equilibrium} by noting that the sums of strength values differ only by $\sigma(d)$, which 
is $0$ by assumption.

Monotony and Strict Monotony follow from \eqref{eq_equilibrium} by observing that additional attackers can only increase the output of
the exponential function in the denominator and thus decrease the strength. Symmetrically, additional supporters can only increase
the strength. In particular, they will increase or decrease the result if their strength is non-zero.

Reinforcement and Strict Reinforcement follow similar by noting that the aggregated sum for $a$ must be less than the aggregated sum
for $b$. Hence, the denominator for $a$ is smaller and thus its strength is larger.

Resilience follows from \eqref{eq_equilibrium} by noting that there is always a finite number of edges so that the outcome of the logistic function
is always strictly between $0$ and $1$.

Franklin follows immediately from \eqref{eq_equilibrium} by noting that the aggregated sums for $a$ and $b$ are equal.

For Weakening, note that the assumptions imply that $\sum\limits_{x \in \attacker(a)}  \sigma(x) - \sum\limits_{x \in \supporter(a)} \sigma(x) > 0$. 
Therefore, we have 
$\sigma(a) < \frac{1}{1 + \frac{1-\baseScore(a)}{\baseScore(a)} \exp(0)} =  \baseScore(a)$.

Strengthening follows symmetrically.
 
To prove Duality, we reorder the terms in the claim and show that $\sigma(a) + \sigma(b) =  \baseScore(a) + \baseScore(b)$.  
To simplify notation, let $A_a =\sum\limits_{x \in \attacker(a)}  \sigma(x) - \sum\limits_{x \in \supporter(a)} \sigma(x)$,
$A_b =\sum\limits_{x \in \attacker(b)}  \sigma(x) - \sum\limits_{x \in \supporter(b)} \sigma(x)$.
Note that the assumptions imply that $A_a = - A_b$. 
Note also that $\baseScore(b) = 0.5 - \epsilon = 1 - (0.5 + \epsilon) = 1 - \baseScore(a)$.
Therefore,
\begin{align*}
&\sigma(a) + \sigma(b) \\
&=  \frac{1}{1 + \frac{1-\baseScore(a)}{\baseScore(a)} \exp(A_a)} +  \frac{1}{1 + \frac{\baseScore(a)}{1-\baseScore(a)} \exp(-A_a)} \\
&= \frac{1 + \frac{\baseScore(a)}{1-\baseScore(a)} \exp(-A_a) + 1 + \frac{1-\baseScore(a)}{\baseScore(a)} \exp(A_a) }
{1 + \frac{\baseScore(a)}{1-\baseScore(a)} \exp(-A_a) +  \frac{1 - \baseScore(a)}{\baseScore(a)} \exp(A_a) + 1} \\
&= 1 = \baseScore(a) + \baseScore(b).
\end{align*} 

For Almost Open-mindedness, note that the assumptions basically say that we add $k$ new attackers (supporters) of $a$ with base score $1$.
Since the new arguments have no parents, their strength is $1$ by Stability. By Directionality, they do not affect the strength of
any of $a$'s parents. Therefore, the new attackers (supporters) will increase (decrease) the sum in the exponential function by $k$. 
Hence, as $k$ goes to infinity, the fraction will go to $0$ ($1$).
\end{proof}

\bibliographystyle{aaai}

\end{document}